\documentclass[preprint,twocolumn, 5p]{elsarticle}

\usepackage{amsmath}
\usepackage{amssymb}
\usepackage{amsthm}
\usepackage{subcaption}
\usepackage{yhmath}
\usepackage{color}
\usepackage{xcolor}
\usepackage[outdir=./plot/output/]{epstopdf}

\usepackage{graphicx}

\DeclareMathOperator{\E}{\mathbb{E}}

\newcommand{\w}{{x}}
\newcommand{\what}{{\hat{x}}}
\newcommand{\bhat}{{\widetilde{x}}}

\newtheorem{theorem}{Theorem}[section]

\newtheorem{assumption}[theorem]{Assumption}

\begin{document}
\title{A Biased Estimator for MinMax Sampling and Distributed Aggregation}

\author[inst1]{Joel Wolfrath}
\author[inst1]{Abhishek Chandra}

\affiliation[inst1]{organization={University of Minnesota},
            city={Minneapolis},
            state={Minnesota},
            country={USA}}

\begin{abstract}
MinMax sampling is a technique for downsampling a real-valued vector which minimizes the maximum variance over all vector components.
This approach is useful for reducing the amount of data that must be sent over a constrained network link (e.g. in the wide-area).
MinMax can provide unbiased estimates of the vector elements, along with unbiased estimates of aggregates when vectors are combined from multiple locations.
In this work, we propose a biased MinMax estimation scheme, \textit{B-MinMax}, which trades an increase in estimator bias for a reduction in variance.
We prove that when no aggregation is performed, B-MinMax obtains a strictly lower MSE compared to the unbiased MinMax estimator.
When aggregation is required, B-MinMax is preferable when sample sizes are small or the number of aggregated vectors is limited. 
Our experiments show that this approach can substantially reduce the MSE for MinMax sampling in many practical settings.
\end{abstract}


\maketitle

\section{Introduction}
Modern applications increasingly generate and persist data in distinct geographical locations.
Analytics over geo-distributed data sources is constrained by wide-area network (WAN) links, which makes it slow and expensive to aggregate data in a centralized location.
Deriving insights from these data sources is an important area of research, and includes techniques such as federated learning, distributed state aggregation, and relational queries.
These techniques still require data to be sent over the network (e.g. model weights) which can increase latency due to the limited WAN throughput.

MinMax sampling can be used to alleviate these issues by reducing the size of the data transfer~\cite{minmax}.
This approach leverages Poisson sampling to send only a representative sample of the data over the network. 
MinMax sampling has already proven useful for federated learning~\cite{mm-fl2} and aggregating distributed relational data~\cite{distributed-state, plexus}.
It has also proven to be more efficient compared to other sampling strategies in the literature~\cite{iceberg,topk1,topk2}.
However, requiring estimates to be unbiased results in an unnecessarily high mean squared error (MSE) in certain settings.
For this reason, we consider introducing a small amount of bias into the estimation scheme to reduce the variance and the overall MSE of the estimator. \\

\textbf{Contributions.} We make the following research contributions:
\begin{itemize}
\item We propose \textit{B-MinMax}, a biased MinMax estimator which achieves a lower MSE compared to MinMax in several settings.
\item We introduce a mechanism which allows B-MinMax to defer to unbiased estimation when preferable.
\item We empirically evaluate the proposed approach across a variety of data distributions and aggregation settings.
\end{itemize}

\noindent
Our experiments show that a principled use of the biased B-MinMax estimator can result in a substantially lower MSE across a variety of applications.

\section{MinMax Estimation}

We are given a set of $k$ vectors, $\mathbf{x}_i \in \mathbb{R}^d$ for $i \in 1..k$, which may represent model weights or some other local value\footnote{This analysis assumes equal vector dimension across sites to simplify the presentation. In practice, the local vectors may have different dimension or correspond to different quantities. This can be accommodated by associating a key with each vector value $\w_{i,j}$, without loss of generality.}.
The task is to obtain a centralized estimate of $\mathbf{x} = \sum_{i=1}^k \mathbf{x}_i$.
Given a sample size $n_i$ associated with each site, MinMax sampling proceeds by locally assigning a probability $p_{i,j}$ to each vector component $\w_{i,j}$, and performing Poisson sampling to determine a subset to send over the network.
The sampling probability is given by:

\begin{equation}
    p_{i,j} = \frac{\w_{i,j}^2}{\w_{i,j}^2 + C_i}
\end{equation}

\noindent
where $C_i$ is selected to satisfy:

\begin{equation}
    n_i = \sum_{j=1}^d \frac{\w_{i,j}^2}{\w_{i,j}^2 + C_i}
\end{equation}

\noindent
which ensures that the expected sample size is $n_i$.
If $\w_{i,j}$ is not selected to be in the sample, it is assumed to be zero.
If $\w_{i,j}$ is selected, the estimator $\what_{i,j}$ is computed and sent over the network, which contains an adjustment to maintain unbiasedness.
More precisely, $\what_{i,j}$ is computed as:

\begin{equation}
    \what_{i,j} =
    \begin{cases}
      \frac{\w_{i,j}}{p_{i,j}}, & \text{with probability}\ \; p_{i,j}\\
      0, & \text{with probability}\ \; 1 - p_{i,j}
    \end{cases}
\end{equation}

\noindent
which yields a corresponding MSE of:
\begin{equation}
\label{mse-minmax}
    \textrm{MSE}(\what_{i,j}) = p_{i,j}\left(\frac{\w_{i,j}}{p_{i,j}} - \w_{i,j} \right)^2 + (1 - p_{i,j})\w_{i,j}^2
\end{equation}

\noindent
It is known that $\what_{i,j}$ is an unbiased estimator for $\w_{i,j}$~\cite{minmax}.
Furthermore, combining values from each site results in an unbiased estimate for components of the aggregate vector $\mathbf{x}$, where $\w_{\cdot j}$ is estimated by:

\begin{equation}
    \what_{\cdot j} = \sum_{i=1}^k \what_{i,j}
\end{equation}

\section{The B-MinMax Estimator}
We propose another estimator, B-MinMax, which makes a small adjustment to MinMax estimation.
Rather than adjusting the weight to be unbiased (i.e. sending $\w_{i,j} / p_{i,j}$) we simply send the original $\w_{i,j}$ value over the network.
Define the B-MinMax estimator to be:
\begin{equation}
    \bhat_{i,j} =
    \begin{cases}
      \w_{i,j}, & \text{with probability}\ \; p_{i,j}\\
      0, & \text{with probability}\ \; 1 - p_{i,j}
    \end{cases}
\end{equation}

\noindent
This approach introduces bias into the original MinMax estimation, but has the potential to reduce the overall MSE.
Figure \ref{fig:bmm} shows an overview of B-MinMax sampling.
The algorithm closely mirrors the MinMax scheme, except samples are not adjusted prior to sending (i.e. we send $\w_{i,j}$ directly rather than $\w_{i,j} / p_{i,j}$).

\begin{figure}
    \centering
    \includegraphics[width=\linewidth]{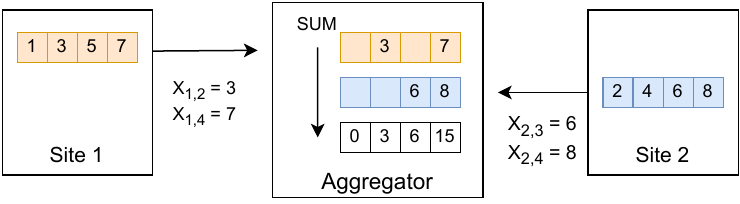}
    \caption{Overview of B-MinMax Sampling with $k=2$ sites and a sample size of $n_1 = n_2 = 2$.}
    \label{fig:bmm}
\end{figure}

\begin{theorem}
    \begin{equation}
    \label{mse-bminmax}
        \textrm{MSE}(\bhat_{i,j}) = \w_{i,j}^2 (1 - p_{i,j})
    \end{equation}
\end{theorem}
    
\begin{proof}
    The estimator bias is given by:
\begin{align*}
    \textrm{Bias}(\bhat_{i,j})
    &= \E[\bhat_{i,j}] - \w_{i,j} \\
    &= \w_{i,j}(p_{i,j} - 1)
\end{align*}

\noindent
The corresponding variance is:
\begin{align*}
    \textrm{Var}(\bhat_{i,j})
    &= p_{i,j}(\w_{i,j} - \w_{i,j}p_{i,j})^2 + (1-p_{i,j})(p_{i,j}\w_{i,j})^2 \\
    &= (p_{i,j} - p_{i,j}^2)\w_{i,j}^2
\end{align*}

\noindent
We therefore obtain the following expression for the MSE:
\begin{align*}
    \textrm{MSE}(\bhat_{i,j})
    &= (p_{i,j} - p_{i,j}^2)\w_{i,j}^2 + \left[ \w_{i,j}(p_{i,j} - 1) \right]^2 \\
    &= \w_{i,j}^2 (1 - p_{i,j})
\end{align*}
\end{proof}

We now explore how the MSE for MinMax and B-MinMax compares across multiple settings.

\subsection{B-MinMax without Aggregation}
If we are collecting data from a single site, the B-MinMax estimator always outperforms the standard MinMax approach in terms of MSE.

\begin{assumption}[]
    \label{ma1}
        We assume $\w_{i,j} \neq 0$. A value of zero may always be estimated with perfect accuracy and no network overhead. \\
    \end{assumption}
    
\begin{assumption}[]
    \label{ma2}
    We assume $p_{i,j} \in (0,1)$.
    This allows us to exclude unrealistic samples sizes, i.e. where $n_i = 0$ or $n_i = d$. \\
\end{assumption}
    
\begin{theorem}
    Under assumptions \ref{ma1} and \ref{ma2}, given a single site $k=1$, $\textrm{MSE}(\bhat_{i,j}) < \textrm{MSE}(\what_{i,j})$.
\end{theorem}
    
\begin{proof}
    Using the results in equations \ref{mse-minmax} and \ref{mse-bminmax}, the difference in MSE can be computed as:
    \begin{align*}
        \textrm{MSE}(\bhat_{i,j}) - \textrm{MSE}(\what_{i,j})
        &= \w_{i,j}^2 \left( \frac{1}{p_{i,j}} + p_{i,j} -2 \right) \\
        &= \frac{\w_{i,j}^2 (p_{i,j}-1)^2}{p_{i,j}} \\
        &> \; 0 
    \end{align*}
\end{proof}

Not only does B-MinMax obtain a strictly smaller MSE for $k=1$, the right-hand limit of the difference in MSE diverges to $\infty$ as $p_{i,j} \rightarrow 0$.
Therefore, the B-MinMax estimator becomes increasingly better (in terms of MSE) as $p_{i,j}$ decreases.

\subsection{B-MinMax with Aggregation}
We have shown that B-MinMax is always preferable to the unbiased MinMax scheme when there is no aggregation involved across multiple sites.
We now explore the trade-offs when using B-MinMax combined with aggregation.

When aggregating data across multiple sites, the expression for the MSE now becomes:
\begin{align}
MSE(\bhat_{\cdot j})
&= \textrm{Var}\left[\sum_{i=1}^k \bhat_{i,j}\right] + \textrm{Bias}\left[\sum_{i=1}^k \bhat_{i,j}\right]^2
\end{align}

\begin{figure}
    \centering
    \begin{subfigure}{.49\linewidth}
      \centering
      \includegraphics[width=\linewidth]{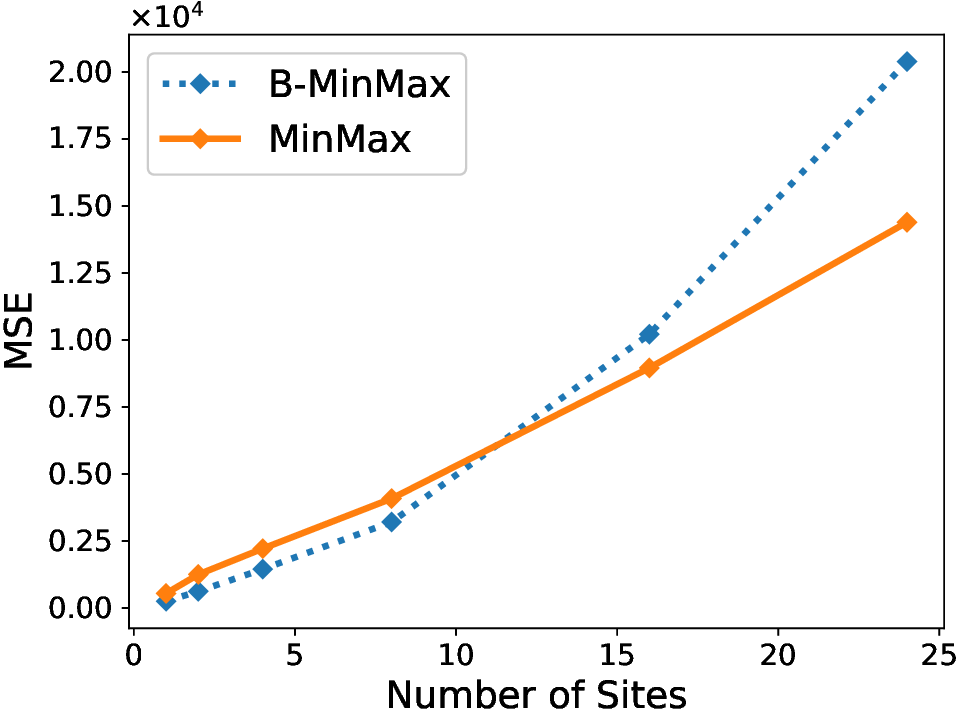}
      \caption{2x compression ratio}
      \label{fig:bmm-2x}
    \end{subfigure}
    \begin{subfigure}{.49\linewidth}
      \centering
      \includegraphics[width=\linewidth]{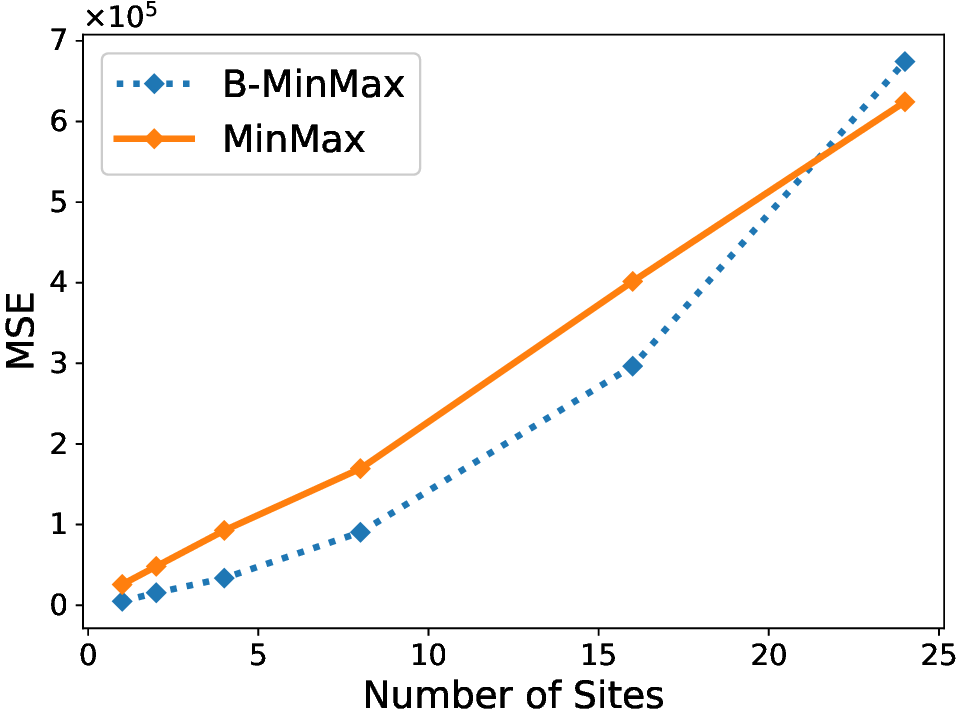}
      \caption{4x compression ratio}
      \label{fig:bmm-4x}
    \end{subfigure}
    \caption{MSE using aggregation across multiple sites, with $\w_{i,j} \sim Zipf(1)$}
    \label{fig:bmm2}
\end{figure}

\noindent
This expression is no longer guaranteed to be smaller than MinMax, due to the squared bias term.
This effect can be observed in figure \ref{fig:bmm2}.
As the number of sites increases, we observe that eventually the MinMax estimator produces a smaller MSE compared to B-MinMax.
We therefore require a mechanism for determining when B-MinMax is preferable. \\

\noindent
\textbf{Adaptive B-MinMax.} 
When aggregation is involved, we propose estimating the MSE of both MinMax and B-MinMax on the fly, and selecting the approach with the smaller MSE.
This adaptive approach can be formulated as follows:

\begin{equation}
    \bhat_{\cdot j} = \sum_{i=1}^k \frac{\bhat_{i,j}}{\theta}
\end{equation}

\noindent
where:

\begin{equation}
    \theta =
    \begin{cases}
      1, & \text{if}\ \textrm{MSE}(\bhat_{\cdot j}) - \textrm{MSE}(\what_{\cdot j}) > 0 \\
      p_{i,j}, & \text{otherwise}
    \end{cases}
\end{equation}

To properly evaluate $\theta$, we require a mechanism for estimating $\textrm{MSE}(\bhat_{\cdot j})$.
We propose sending two additional data points from each site: $\bar{v}_i$ and $\bar{b}_i$, the average variance and bias across all local elements at site $i$.
More specifically, let:
\begin{equation}
    \bar{v}_i = \frac{1}{d} \sum_{j=1}^d \textrm{Var}(\bhat_{i,j})
\end{equation}

\noindent
and
\begin{equation}
    \bar{b}_i = \frac{1}{d} \sum_{j=1}^d \; \mid \textrm{Bias}(\bhat_{i,j})  \mid
\end{equation}
\noindent
Then, we can estimate $\textrm{MSE}(\bhat_{\cdot j})$ as:
\begin{equation}
    \widehat{MSE}(\bhat_{\cdot j}) = \textrm{Var}\left[\sum_{i=1}^k \bar{v}_{i}\right] + \textrm{Bias}\left[\sum_{i=1}^k \bar{b}_{i}\right]^2 
\end{equation}


\noindent
Note that this approach can be implemented with constant network overhead for each site, since we only require two additional data points ($\bar{v}_i$ and $\bar{b}_i$) to be sent over the network.

\section{Empirical Evaluation}

To validate our approach, we conduct experiments with both synthetic and real-world data.
We test across a variety of factors, including different compression ratios the amount of aggregation (number of sites).
In each case, we compare our adaptive B-MinMax scheme against the MinMax baseline.
We consider two different datasets: a set of synthetic vectors following a Zipfian distribution and vectors containing the trained model weights for a ResNet18 model~\cite{resnet}.
Unless otherwise stated, we aggregate vectors across $k=4$ sites with $d=10000$ elements and a target compression ratio of 4.

\subsection{Compression Ratio}

\begin{figure}
    \centering
    \begin{subfigure}{.49\linewidth}
      \centering
      \includegraphics[width=\linewidth]{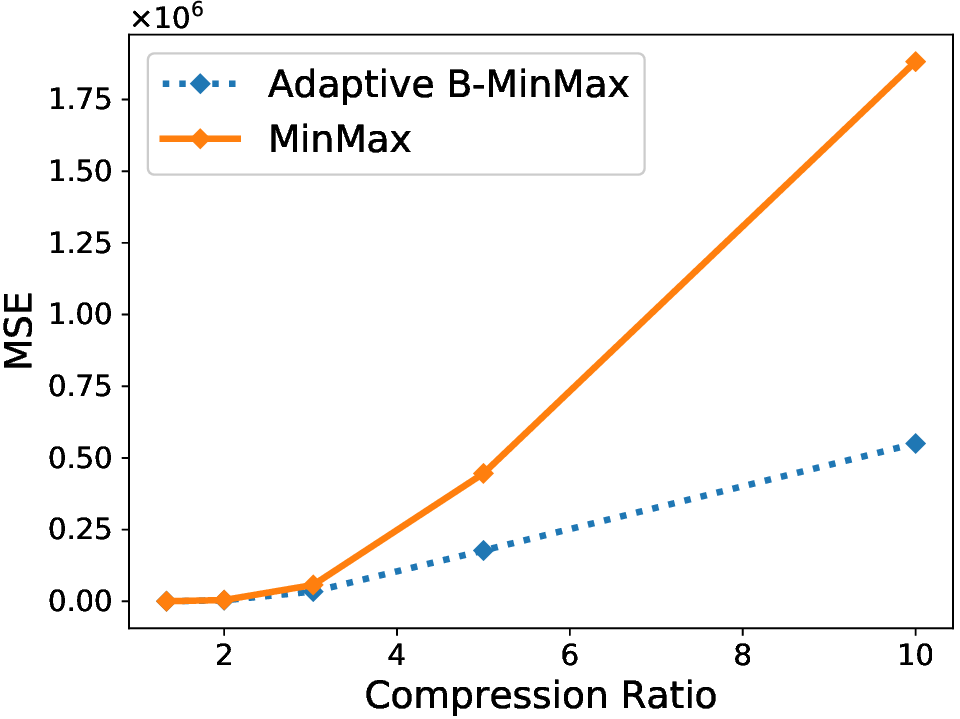}
      \caption{Zipf(1) Distribution}
      \label{fig:compress-zipf}
    \end{subfigure}
    \begin{subfigure}{.49\linewidth}
      \centering
      \includegraphics[width=\linewidth]{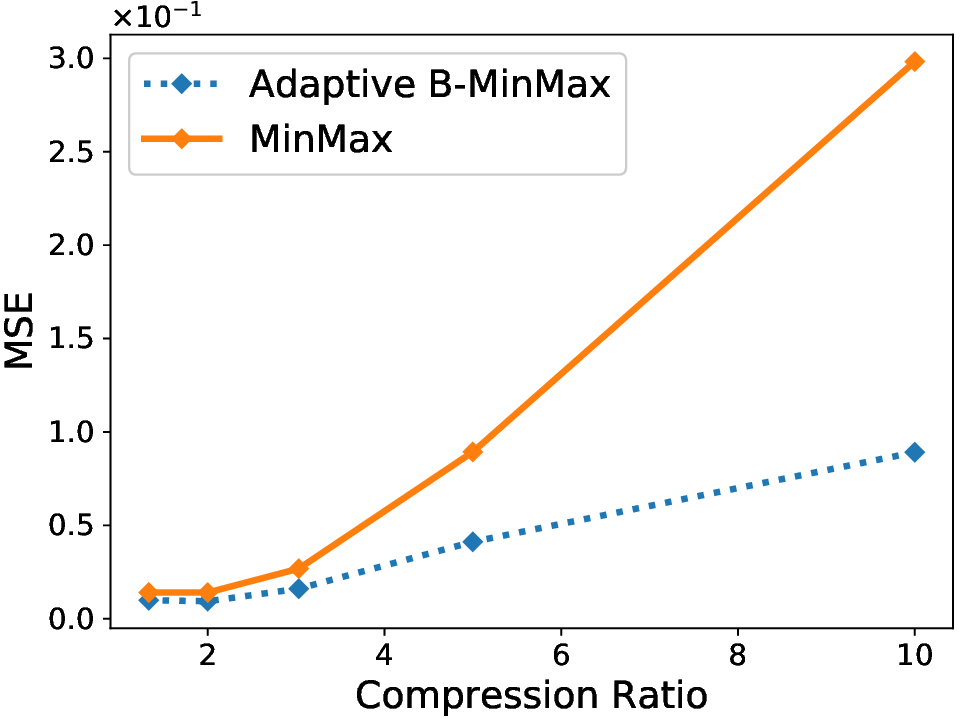}
      \caption{Resnet18 Weights}
      \label{fig:compress-other}
    \end{subfigure}
    \caption{MSE across different compression ratios, with $k=4$.}
    \label{fig:compress}
\end{figure}

We first evaluate the performance of adaptive B-minmax across various compression ratios (i.e. sample sizes).
Figure \ref{fig:compress} shows the results for this experiment.
We observe that as the compression ratio increases (i.e. sample size decreases), the relative performance of B-MinMax increases.
Using the ResNet18 data (figure \ref{fig:compress-other}), we observe a 54\% reduction in MSE at the 5x compression ratio and a 70\% reduction at the 10x ratio.

\subsection{Number of Sites}

\begin{figure}
    \centering
    \begin{subfigure}{.49\linewidth}
      \centering
      \includegraphics[width=\linewidth]{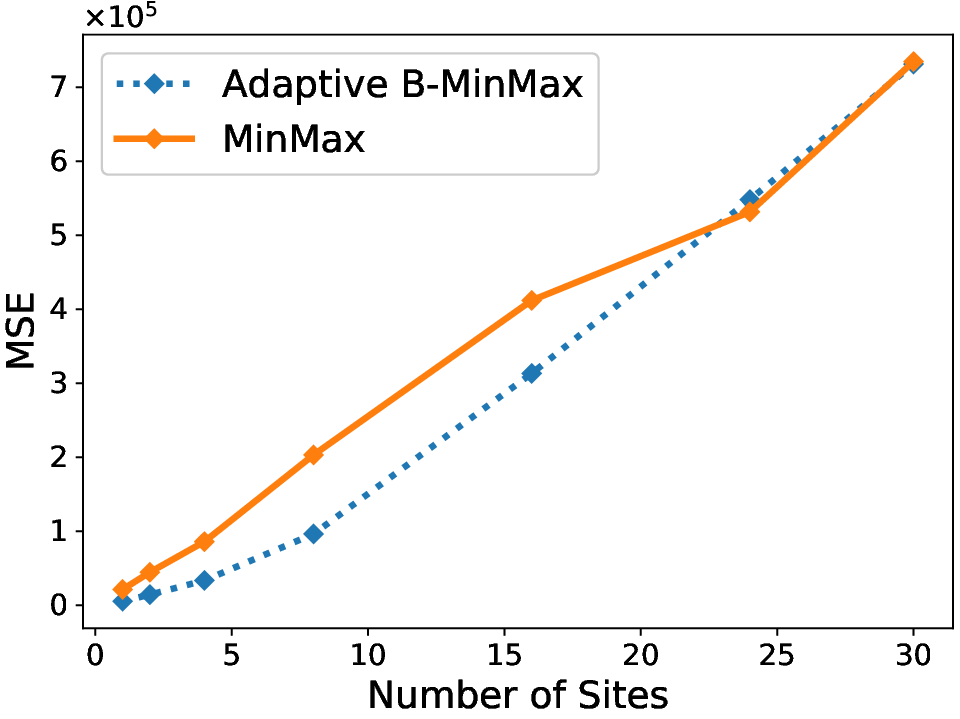}
      \caption{Zipf(1) Distribution}
      \label{fig:sites-zipf}
    \end{subfigure}
    \begin{subfigure}{.49\linewidth}
      \centering
      \includegraphics[width=\linewidth]{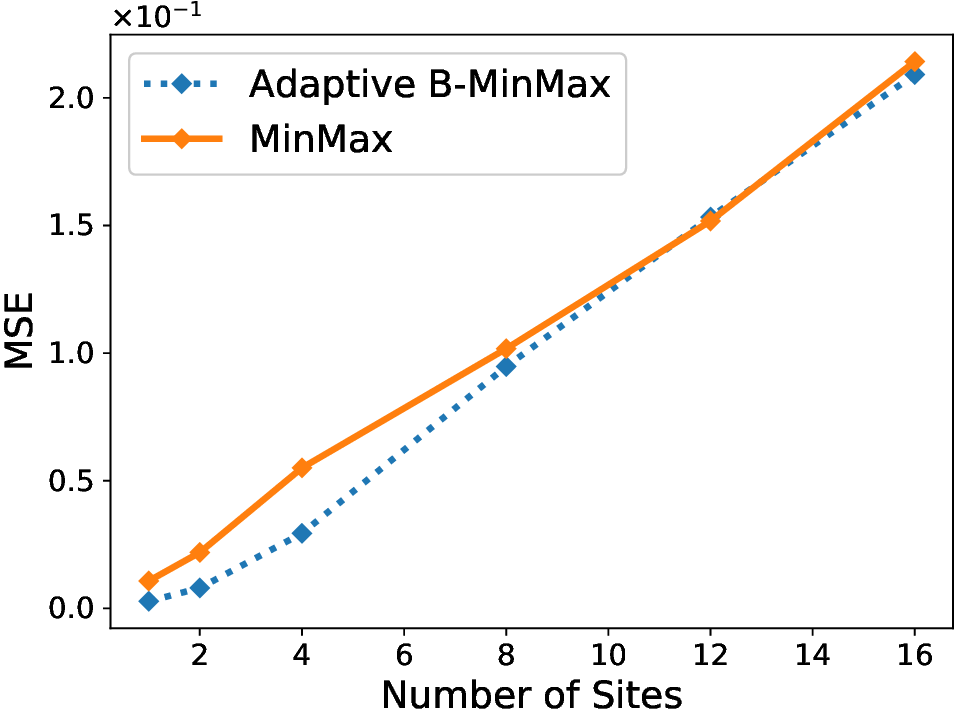}
      \caption{Resnet18 Weights}
      \label{fig:sites-other}
    \end{subfigure}
    \caption{MSE across different numbers of sites, with a compression ratio of 4x.}
    \label{fig:sites}
\end{figure}

Figure \ref{fig:sites} shows how adaptive B-minmax performs across an increasing number of sites.
In both experiments, we observe an improvement in terms of MSE for a small number of sites.
Eventually adaptive B-MinMax detects the MSE will be large, and defers to the unbiased approach.
For example, on the Zipfian data in figure \ref{fig:sites-zipf}, we observe this effect happens after aggregating data from 25 sites.



\subsection{Discussion}
The following factors highlight the situations in which B-MinMax offers an improvement in terms of MSE:
\begin{itemize}
    \item A large compression ratio is required (i.e. small sample sizes)
    \item The number of sites aggregated is small
\end{itemize}
Federated learning \textit{cross-silo}, rather than cross-device, could benefit from this approach, since the number of participating locations is small~\cite{silo3, silo1, silo2}.
Clustered federated learning could benefit, where only data within a cluster is aggregated.
We also note that if sites have very few keys in common, this effectively reduces the amount of aggregation, which is a benefit for B-MinMax.
Therefore, aggregating key-value pairs where some keys are unique to each location could benefit.

\section{Conclusion}
We proposed B-MinMax: a biased MinMax estimator which achieves a lower MSE in several settings.
We showed that B-MinMax is preferable with smaller sample sizes, few shared keys, or when aggregating across a small number of sites.
When the bias becomes too high, applications can dynamically adapt and prefer the unbiased MinMax approach. 


\bibliographystyle{elsarticle-num}
\bibliography{refs}

\end{document}